\def\year{2019}\relax
\documentclass[letterpaper]{article} 
\usepackage{aaai19}
\usepackage{times}
\usepackage{helvet}
\usepackage{courier}
\usepackage[hyphens]{url}
\usepackage{graphicx}
\usepackage{graphicx}
\usepackage{amssymb}
\usepackage{amsmath}

\usepackage[vlined, ruled, boxed,linesnumbered]{algorithm2e}
\SetKwInOut{Parameters}{Parameters}
\SetKwBlock{Setup}{Setup}{}
\SetKwFor{Loop}{Loop}{}

\usepackage[utf8]{inputenc}
\usepackage[english]{babel}
\usepackage{amsthm}
\theoremstyle{definition}

\newtheorem*{definition*}{Definition}
\newtheorem{lemma}{Lemma}

\usepackage{multirow}

\begin{document}

\title{Efficiently Exploring Ordering Problems through Conflict-directed Search}

\author{Jingkai Chen \and Cheng Fang \and David Wang \and Andrew Wang \and Brian Williams\\
MIT Computer Science and Artificial Intelligence Laboratory\\
}

\maketitle
\begin{abstract}
\begin{quote}
In planning and scheduling, solving problems with both state and temporal constraints is hard since these constraints may be highly coupled. Judicious orderings of events enable solvers to efficiently make decisions over sequences of actions to satisfy complex hybrid specifications. The ordering problem is thus fundamental to planning. Promising recent works have explored the ordering problem as search, incorporating a special tree structure for efficiency. However, such approaches only reason over partial order specifications. Having observed that an ordering is inconsistent with respect to underlying constraints, prior works do not exploit the tree structure to efficiently generate orderings that resolve the inconsistency. In this paper, we present Conflict-directed Incremental Total Ordering (CDITO), a conflict-directed search method to incrementally and systematically generate event total orders given ordering relations and conflicts returned by sub-solvers. Due to its ability to reason over conflicts, CDITO is much more efficient than Incremental Total Ordering. We demonstrate this by benchmarking on temporal network configuration problems that involve routing network flows and allocating bandwidth resources over time.
\end{quote}
\end{abstract}

\section{Introduction}
Adaptive Large Neighborhood Search has shown impressive results in vehicle routing problems with time windows, which iteratively destroys and repairs the total order of tasks to achieve high-quality plans by using heuristics \cite{ropke2006adaptive}. Timeline-based planners such as tBurton use ordering algorithms to unify multiple timelines such that sub-solvers can efficiently check the plan's consistency \cite{wang2015tburton}.

Similar to tBurton, many hybrid planners are designed in a hierarchical architecture where abstract tasks are generated by heuristic search or partial-order planning algorithms and then refined into more concrete courses of actions by resource managers or schedulers. There has been considerable progress in the development of these hybrid planners for solving problems with temporally evolving numeric effects, complex objective functions, automatic timed transitions, or temporal uncertainty \cite{coles2012colin,barreiro2012europa,fernandez2018scottyactivity,wang2015tburton,umbrico2018timeline}. As these abstract tasks are always partially ordered, these planners impose either total or partial orders on the tasks in order to enable grounded consistency check (e.g., temporal consistency) or refinement (e.g., motion trajectory generation). This motivates us to develop an algorithm to efficiently order these abstract tasks over time by interacting with different underlying solvers such that a consistent plan exists under this ordering.

In this paper, we introduce Conflict-directed Incremental Total Ordering (CDITO), a systematic and conflict-directed search method to generate consistent total orders of the start and end events of abstract tasks. CDITO starts with a total order of these events and incrementally alters inconsistent partial orders by reasoning over violated ordering relations or inconsistency discovered by the sub-solvers in terms of resource capacities or event schedules.

CDITO is built upon the idea of searching the total order tree introduced in \cite{ono2005constant}. In this tree, total orders are arranged in a special structure such that some subtrees can be pruned with respect to violated partial orders. Another successful work based on \cite{ono2005constant} is Incremental Total Ordering (ITO), which is used in tBurton to unify multiple timelines \cite{wang2015}. ITO plays an important role in improving the time and space efficiency of tBurton by using a queue to store operations of altering partial orders. We are inspired by ITO to reason over these operations and extend it to account for inconsistency in sub-problems.

For total orders found to be inconsistent, we may reason over which partial orders led to this inconsistency. This allows us to discover conflicting partial orders, which must be resolved. As we are working over a total order tree, this problem is one of conflict-directed search, and thus we may leverage the insights of \cite{williams2007conflict} to focus our search by pruning search space. In addition to common conflict-directed search techniques, CDITO leverages the special structure of the total order tree to quickly jump to promising candidate orders.

The idea of interacting with sub-solvers is also used in Satisfiability Modulo Theories (SMT) solvers such as Z3 \cite{de2008z3}. As SMT solvers determine satisfiability of formulas with respect to some background theories, it can be used to address complex systems of constraints. Similar to SMT solvers, CDITO is able to interact with sub-solvers through propositional combinations of partial orders, which enables our approach to leverage advanced underlying solvers supporting various user-defined features for expressivity. 

The remainder of this paper is organized as follows. We first introduce a motivating example, a temporal network configuration problem that involves routing network flows and allocating bandwidth resources with respect to the requirements such as loss, delay, and throughput (Section~\ref{section:example}). In Section~\ref{section:def}, we give the formal definition of our ordering problem. We also formulate the motivating example as an ordering problem in this section. In Section~\ref{section:approach}, we introduce the total order tree and a search strategy within this tree. We also present the methods to extract implicit ordering relations and resolve ordering conflicts. Then, we introduce the algorithmic details of CDITO. In Section~\ref{section:experiment}, we benchmark CDITO against ITO on temporal network configuration problems with different size and complexity and discuss the empirical results.

\section{Motivating Example}\label{section:example}
Consider a network configuration management problem in which we need to schedule and route three network flows and allocate bandwidth resources of a network. In this network, the links have different characteristics of loss, delay, and bandwidth capacity, as shown in Figure~\ref{fig:example_topology}.

\begin{figure}[ht]
\centering
\includegraphics[width=0.45\columnwidth]{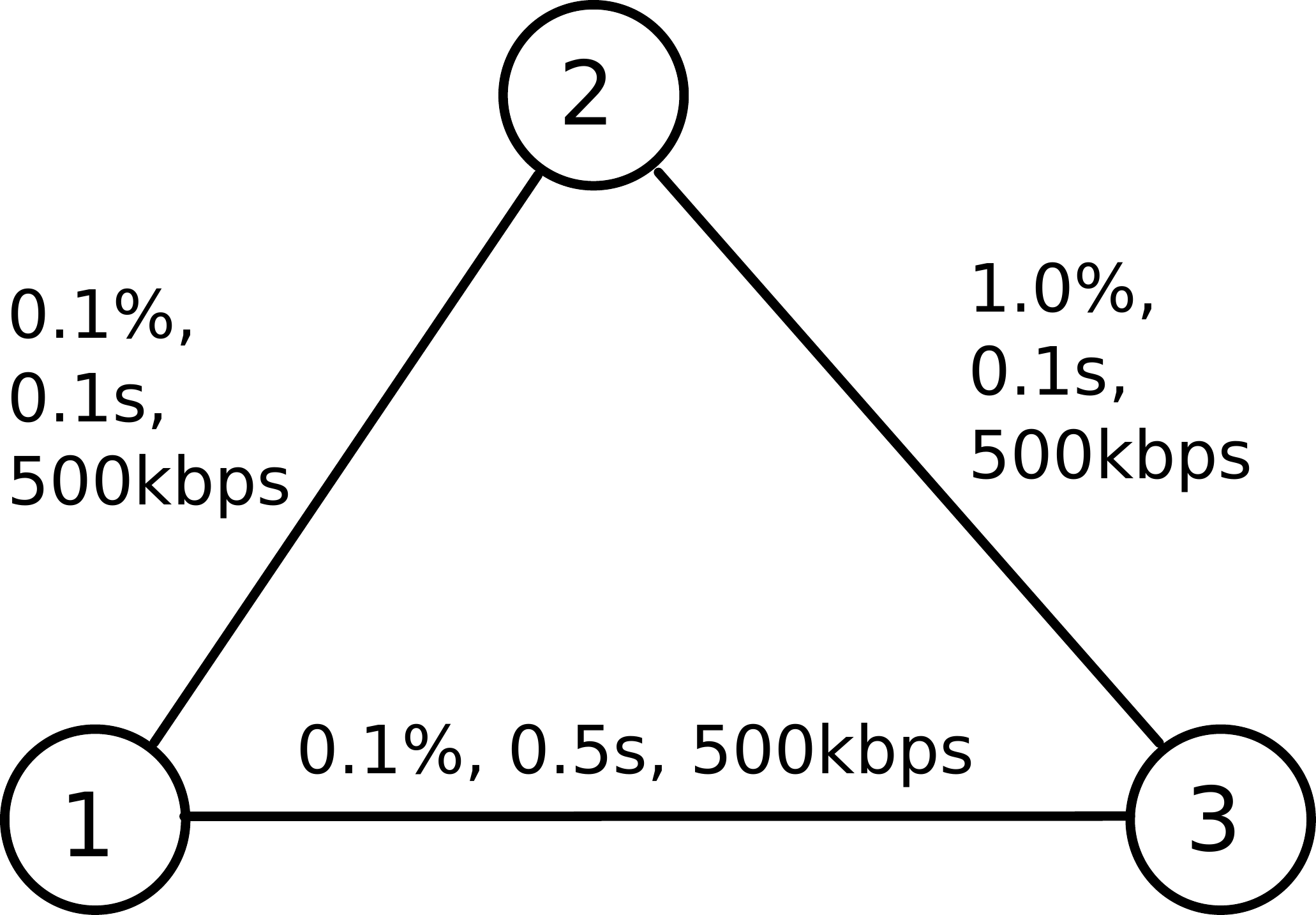}
\caption{Network topology.}
\label{fig:example_topology}
\end{figure}

Table~\ref{tab:example_flow} gives the mission specifications of these three flow on source nodes, destination nodes, maximum loss, maximum delay, minimum required throughput, and allowable duration lengths. The mission also has four temporal requirements: (1) Flow-B and Flow-C should start immediately when the mission begins; (2) Flow-B and Flow-C should end at least 20 seconds apart; (3) Either Flow-B or Flow-C should end before Flow-A starts; (4) the whole mission should take less than 70 seconds.

\begin{table}[ht]\scriptsize
\begin{center}
\begin{tabular}{|c||c|c|c|c|c|c|}
 \hline
 Flow & Source & Sink & Loss & Delay & Throughput & Duration\\
 \hline
 Flow-A & 1 & 2 & 0.5\% & 1s & 200kbps & [30,60]s\\
 Flow-B & 1 & 2 & 3\% & 1s & 360kbps & [30,60]s\\
 Flow-C & 1 & 2 & 3\% & 0.3s & 360kbps & [30,60]s\\
 \hline
\end{tabular}
\caption{Mission specifications of flows.}
\label{tab:example_flow}
\end{center}
\end{table}

\begin{figure}[ht]
\centering
\includegraphics[width=0.98\columnwidth]{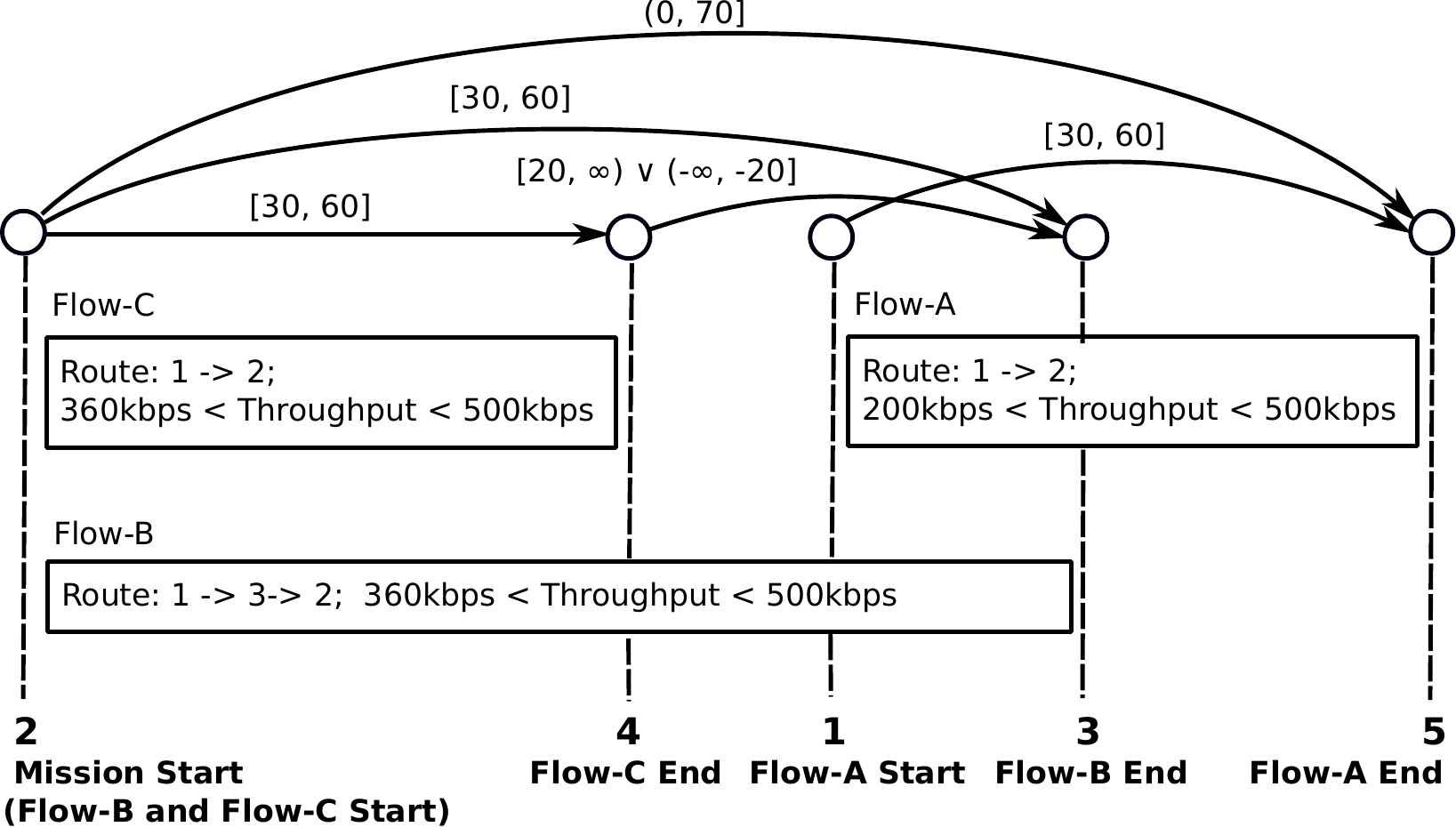}
\caption{Solution of the motivating example.}
\label{fig:example_soln}
\end{figure}

There is only one consistent total order for all the events. This total order and corresponding routes are given in Figure~\ref{fig:example_soln}. We also show the temporal constraints of the example in Figure~\ref{fig:example_soln}. Given the duration requirements on these flows and the temporal requirements (1) - (2), we know Flow-A has to overlap with either Flow-B or Flow-C. Given the loss and delay constraints, we know the only feasible route for Flow-A and Flow-C is Path 1-2, while Flow-B can choose Path 1-2 or Path 1-3-2. However, given the bandwidth capacity constraint, Path 1-2 cannot transfer two flows concurrently, and thus Flow-A and Flow-C cannot overlap. Therefore, Flow-A must start after Flow-C ends and before Flow-B ends, which leads to the total order and routes in Figure~\ref{fig:example_soln}. As we can see, this problem involves routing flows and allocating resources with respect to multiple characteristics, which is hard especially when the mission is over a large network and involves multiple flows \cite{chen2018radmax}. The decision-making problem at a time point is already NP-hard since it contains the traveling salesman problem as a special case. Our motivating example is even more complex since we need to schedule flows and make these decisions over time under disjunctive temporal constraints. We show that our method can efficiently solve this problem by reasoning over the ordering information extracted from the mission specifications and temporal requirements on demand.

\section{Problem Formulation}\label{section:def}
Our ordering problem is given by a tuple $\langle E,\Phi, h \rangle$:
\begin{itemize}
    \item $E$ is a set of $n$ events represented by the natural numbers $\{1,2,..,n\}$.
    \item $\Phi$ is an ordering relation that is a set of clauses, and a disjunct of each clause is a partial order $a \prec b$ that constrains $a \in E$ to precede $b \in E$.
    \item $h:2^\mathcal{L} \rightarrow \langle \{\top, \bot\}, 2^{\mathcal{C}_h} \rangle$ is a user-defined  consistency function that maps a total order $\mathcal{L}$ of $E$ to a Boolean value indicating the consistency of $\mathcal{L}$, and a set of ordering conflicts $\mathcal{C}_h$. 
    Each conflict $c \in \mathcal{C}_h$ is a conjunction of partial orders that are implied by $\mathcal{L}$ and result in inconsistency.
\end{itemize}

A candidate solution of this ordering problem is a total order $\mathcal{L}$ that is a sequence of all the events of $E$. $\mathcal{L}$ is a solution if and only if $\mathcal{L}$ satisfies all the clauses in $\Phi$ and is checked to be consistent by $h$. 

Note that we do not allow events to co-occur; thus, we require a strict ordering such that $\lnot (a \prec b) = (b \prec a)$. 

In general, $h$ can be any evaluable function that returns a consistency indicator and ordering conflicts $\mathcal{C}_h$ from which implicit ordering relations can be extracted. These relations are not included in $\Phi$ but should hold in terms of the consistency defined by $h$. 

Our motivating example can be formulated as follows:
\begin{itemize}
    \item $E = \{1,2,3,4,5\} \equiv$ \{Flow-A Start, Mission Start, Flow-B End, Flow-C End, Flow-A End\}  as shown in Figure~\ref{fig:example_soln}.
    \item $\Phi = \{ \varphi_1, \varphi_2, \varphi_3, \varphi_4 \}$, where $\varphi_1 = 1 \prec 5$, $\varphi_2 = 2 \prec 3$, and $\varphi_3 = 2 \prec 4$ constrain each flow's start to precede its end, and $\varphi_4 = (3 \prec 1) \lor (4 \prec 1)$ captures the temporal requirement (3).
    \item $h$ is able to take as input a total order and determine whether a valid plan that satisfies all problem requirements exits under this total order. If such a valid plan does not exist, ordering conflicts are returned. To determine the consistency of a total order, we use the CP solver introduced in \cite{chen2018radmax} and Incremental Temporal Consistency \cite{shu2005enabling} to check state and temporal consistency of the corresponding plan respectively.
\end{itemize}
Along the search, two other clauses will be extracted from the ordering conflicts returned by $h$: $\varphi_5 = (4 \prec 1) \lor (5 \prec 2)$ is extracted from state inconsistency and means Flow-A and Flow-C cannot overlap; $\varphi_6 = (1 \prec 3) \lor (1 \prec 4)$ is extracted from temporal inconsistency and means Flow-A must start before at least one other flow ends. The method to extract implicit ordering relations is introduced in Section~\ref{section:approach:extraction}.

\section{Approach}\label{section:approach}
In this section, we present the design and implementation of the CDITO algorithm, which incrementally and systematically generates total orders by applying conflict-directed search. CDITO leverages the ideas and methods from total order generation and conflict-directed search in the literature; it uses the ideas of the total order tree from \cite{ono2005constant} and conflict resolution from \cite{williams2007conflict} to efficiently explore this total order tree. We start by reviewing the structure of the total order tree and the search strategy within this tree. Then, we move on to the methods to extract implicit ordering relations and resolve ordering conflicts, and the CDITO algorithm.

\subsection{Total Order Tree}\label{section:approach:tree}
In the total order tree of $E = \{1,2,..,n\}$, nodes are total orders of $E$, and edges are operations of altering partial orders of total orders. This tree is rooted at the root total order $\mathcal{L}_r = (1,2,..,n)$ and constructed by expanding all the children of each total order. This tree expansion uses the notion of total order level that is defined by Ono:

\begin{definition*}{(Level)}\label{def:level}
The level of a total order $\mathcal{L} = (p_1, p_2,..,p_n) \not = \mathcal{L}_r$ is the minimum integer $l$ such that $p_l \not = l$. The level of $\mathcal{L}_r$ is $n$.
\end{definition*}

Consider a total order $\mathcal{L} = (p_1, p_2,..,p_n)$ with level $l$. In order to generate all its children with level $1 \leq i < l$, the method from Ono deletes $p_i = i$ from $\mathcal{L}$ and then inserts it in somewhere $(p_{i+1}, p_{i+2},..,p_n)$ such that the level of this child is $i$. To generate all its children, this process is repeated for every $1 \leq i < l$. The completeness proof of this method is given in \cite{ono2005constant}. In this paper, we formally define this operation of moving an event within a total order as an order move:

\begin{definition*}{(Order Move)}
An order move $(i \rightarrow j)$ deletes $p_i$ from a total order $\mathcal{L} = (p_1, p_2,..,p_n)$ and inserts it right after $p_j$ to obtain a total order $\mathcal{L}'$. This operation is denoted as $\mathcal{L}' = \mathcal{L} \oplus (i \rightarrow j)$.
\label{def:move}
\end{definition*}

Note that each edge in the tree is an order move, and only a subset of order moves are represented in the tree. A feasible order move $(i \rightarrow j)$ from $\mathcal{L}$ with level $l$ should satisfy two conditions: (1) $i < l$: a feasible order move should only move an event that is less than $l$; (2) $i < j$: this move should only right shift an event. Therefore, a total order with level $l$ has $(1 + 2 + .. (l-1)) = l(l-1)/2$ children. The total order tree of $E = \{1,2,3,4\}$ is given as an example in Figure~\ref{fig:tree}.

As the total order tree constrains feasible order moves with respect to the total order's level, an important property of this tree can be obtained as Lemma~\ref{lemma:fixed_in_children}: 
\begin{lemma}\label{lemma:fixed_in_children}
For a total order with level $l$, the partial orders between the events $\{l, l+1,..,n\}$ remain in its descendants.
\end{lemma} 
\begin{proof}
Since child generation does not allow moving events that are larger than the parent's level, $\{l, l+1,..,n\}$ are not moved, and their partial orders remain in this total order's children. Given that all descendants' levels are less than $l$, the partial orders between the events $\{l, l+1,..,n\}$ remain in the descendants of this total order as well.
\end{proof}



\begin{figure}[ht]
\centering
\includegraphics[width=0.98\columnwidth]{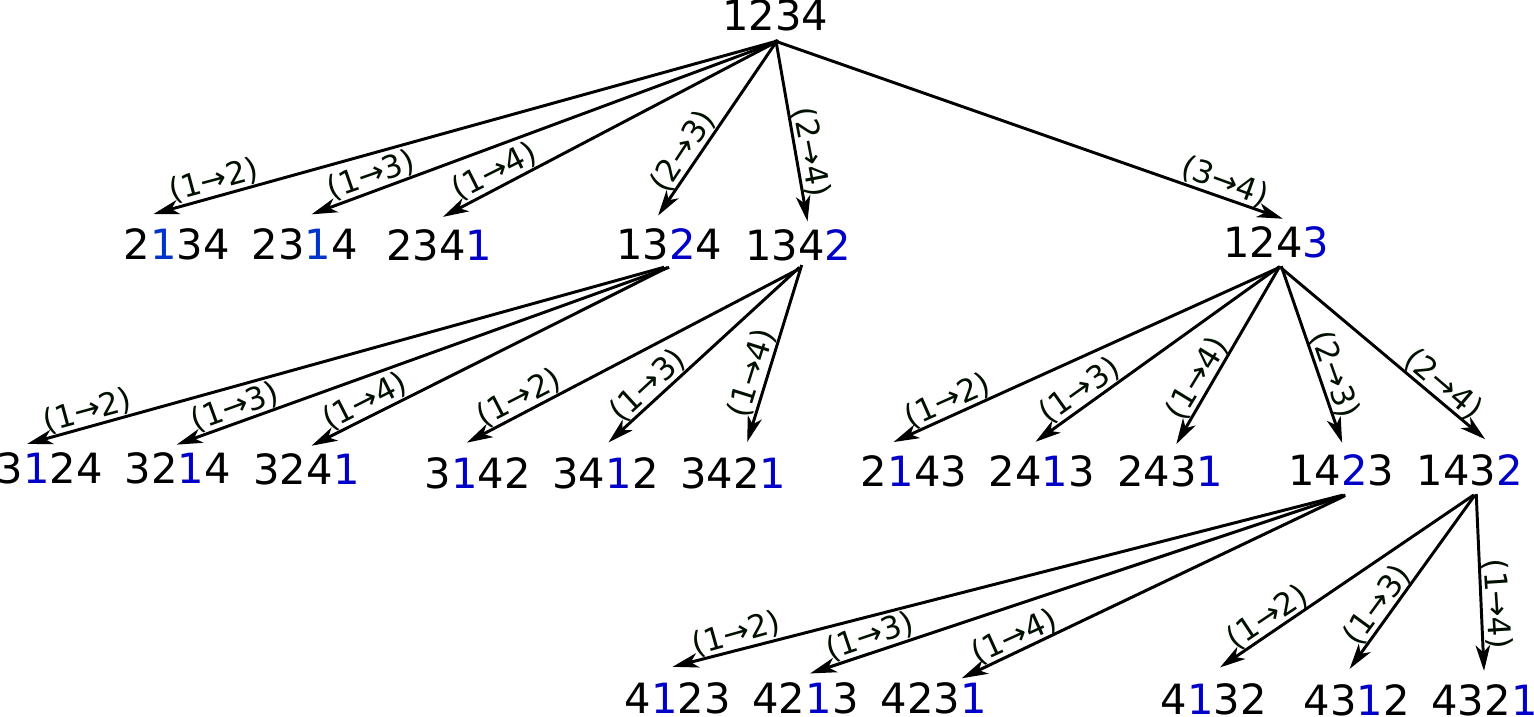}
\caption{Total order tree of $E = \{1,2,3,4\}$. The nodes are total orders; the edges are order moves; the levels of total orders are in blue.}
\label{fig:tree}
\end{figure}

\subsection{Total Order Search}
Now we introduce the search strategy within the total order tree, which basically follows these three rules: (1) children are visited before siblings, which means our strategy is a depth-first search. For example, in Figure~\ref{fig:tree}, after $1324$ is visited, the search will visit its child $3124$ instead of its sibling $1342$; (2) the algorithm backtracks when children are exhausted or there is no child. For example, after the search visits $3421$, it will backtrack to $1342$; (3) from a total order, the group of children with the lowest level $i$ is generated first, and within a group, children are generated by right shifting $i$ until the right end. For example, the children of $1243$ are visited in the order of $2\underline{1}43$, $24\underline{1}3$, $243\underline{1}$, $14\underline{2}3$, and $143\underline{2}$, and the corresponding order moves are $(1 \rightarrow 2)$, $(1 \rightarrow 3)$, $(1 \rightarrow 4)$, $(2 \rightarrow 3)$, and $(2 \rightarrow 4)$.

\begin{algorithm}[ht]\small
\caption{Total Order Search}
\label{alg:tos}
\KwIn{$n$}
\KwOut{$\mathcal{O}$}
$\mathcal{L} \gets (1,2,...,n)$ \;
$\mathcal{O} \gets [\mathcal{L}]$ \;
$\mathcal{P} \gets [(1,1,n)]$ \; \label{line:tos:init_stack}
\While{$\mathcal{P} \not = \{\}$}{
$( i,j,l ) \gets \mathcal{P}[1] $ \;\label{line:tos:read_stack}
\eIf{$j \not = n$ \label{line:tos:next-move-init-start}}
    {$(i^{\dagger},j^{\dagger}) \gets (i,j + 1)$}
    {$(i^{\dagger},j^{\dagger}) \gets (i+1,i+2)$ \label{line:tos:next-move-init-end}}
\eIf{$i^{\dagger} < l$\label{line:tos:maxl}}
    {
    $\mathcal{P}[1] \gets (i^{\dagger}, j^{\dagger}, l )$\; \label{line:tos:update_p_stack}
    $\mathcal{L} \gets \mathcal{L} \oplus (i^{\dagger} \rightarrow j^{\dagger})$ \;
    push $\mathcal{L}$ to $\mathcal{O}$ \;
    push $(1, 1, i^{\dagger})$ to $\mathcal{P}$ \; \label{line:tos:update_c_stack}}
    {pop $\mathcal{P}$\; \label{line:tos:popstack}
    \If{$\mathcal{P} \not = \{\}$}{
        $(i',j',l') \gets \mathcal{P}[1]$ \;
        $\mathcal{L} \gets \mathcal{L} \oplus (j' \rightarrow (i' - 1))$\;\label{line:tos:backinsert}}}}
\KwRet{$\mathcal{O}$}
\end{algorithm}

Total Order Search (Algorithm~\ref{alg:tos}) takes as input the number of events $n = |E|$ and outputs all the total orders $\mathcal{O}$ by following the aforementioned three rules. In Total Order Search, we use the search status $(i,j,l)$ of a total order $\mathcal{L}$ to compute the next move, where $(i \rightarrow j)$ records the last applied order move from $\mathcal{L}$, and $l$ is the level of $\mathcal{L}$. We use a stack $\mathcal{P}$ to store all the search statuses from the root total order to the current total order, and we initialize it with $[(1,1,n)]$ (Line~\ref{line:tos:init_stack}). An example in Figure~\ref{fig:tree} is that, after the search visits $2431$ and backtracks to $1243$, the search status of $1243$ and $1234$ are $(1,4,3)$ and $(3,4,4)$ respectively. Thus, we have $P = [(1,4,3),(3,4,4)]$. Note that $\mathcal{P}[1]$ is the status of the current total order. To generate a child, our algorithm reads $(i,j,l)$ from $\mathcal{P}$ (Line~\ref{line:tos:read_stack}) and computes the next order move  
$(i^{\dagger} \rightarrow j^{\dagger})$ (Lines~\ref{line:tos:next-move-init-start}-\ref{line:tos:next-move-init-end}). If $i^{\dagger} < l$, we apply this order move to obtain a child (Lines \ref{line:tos:update_p_stack}-\ref{line:tos:update_c_stack}), otherwise the search backtracks (Lines \ref{line:tos:popstack}-\ref{line:tos:backinsert}). 

As introduced in Section~\ref{section:approach:tree}, from a total order $\mathcal{L}$ with level $l$, there are $l(l-1)/2$ feasible order moves $(i \rightarrow j)$. Lines~\ref{line:tos:next-move-init-start}-\ref{line:tos:next-move-init-end} enforce these order moves to be sorted in the ascending order of $(ni+j)$, which follows rule (3).

Our search order is different from the breadth-first search order used in \cite{ono2005constant}. We explore the tree in the same order as ITO \cite{wang2015}. However, while ITO uses a queue to store the next several order moves, we use a stack of search statuses to record the explored search space. This will be seen to be critical to manipulating the search order and pruning the search space.

\subsection{Extracting Implicit Ordering Relations}\label{section:approach:extraction}
In this section, we introduce the methods to find ordering conflicts from grounded consistency check, such as state and temporal consistency. Then, we extract implicit ordering relations from these conflicts. 

In real-world problems, the requirements are not restricted to ordering relations. A total order that satisfies the ordering relation $\Phi$ may fail to achieve state or temporal consistency. Therefore, the user-defined consistency function $h$ is used to check the consistency of a total order $\mathcal{L}$ to discover ordering conflicts from which implicit ordering relations can be extracted. It is computationally expensive to enumerate these implicit relations and include them in the ordering relation $\Phi$ in the beginning. Instead, these implicit relations can be found by calling $h$ on demand and then added $\Phi$ . 

We take state and temporal inconsistency as examples to show how to extract ordering conflicts and implicit ordering relations with respect to user-defined consistency.

\subsubsection{State Inconsistency}
We first introduce how to extract ordering and implicit ordering relations from state inconsistency, such as exceeding resource capacities. In our motivating example, Flow-A and Flow-C cannot be transferred concurrently because the only available path that satisfies their loss and delay constraints is Path 1-2. However, the bandwidth capacity of Link 1-2 is not enough to transfer both of them together. Therefore, the concurrency of these two flows leads to an ordering conflict. As Flow-A starts at $1$ and ends at $5$, and Flow-C starts at $2$ and ends at $4$, this ordering conflict can be represented as $c_5 = (1 \prec 4) \land (2 \prec 5)$. This conflict compactly captures all the combinations of this concurrency: $1254$, $1245$, $2154$, and $2145$, which means the start events of these flows happen before the end events. The implicit ordering relation discovered from this ordering conflict should prevent Flow-A and Flow-C from being transferred concurrently. Therefore, this relation is the negation of $c_5$ that is $\varphi_5 = \lnot c_5 = \lnot ((1 \prec 4) \land (2 \prec 5)) = (4 \prec 1) \lor (5 \prec 2)$.

We represent the ordering conflict of the concurrency between two tasks as a conjunction of two partial orders: $R^s_{ij} = (x_i^\vdash \prec x_j^\dashv) \land (x_j^\vdash \prec x_i^\dashv)$, where $i$ and $j$ are the indices of these two tasks; $x_i^\vdash$ and $x_j^\vdash$ are the start events; $x_i^\dashv$ and $x_j^\dashv$ are the end events. When multiple tasks are concurrent, the ordering conflict $c^s$ is as follows: 
\begin{equation}\small
\label{eq:inconsistent_state}
    c^{s} = \underset{i,j}{\land}R^{s}_{ij} = \underset{i,j}{\land} (x_i^\vdash \prec x_j^\dashv).
\end{equation}
where each $R^{s}_{ij}$ represents the concurrency of two tasks. 

Assume that $m$ tasks are concurrent, Equation~\ref{eq:inconsistent_state} can capture this concurrency with a conjunction of $m(m-1)$ partial orders. In every total order featuring this concurrency, if we halve the involved events into two groups with respect to this total order, all the start events are in the first group, and all the end events are in the second group. Thus, this concurrency can be described by specifying all the precedence relations between start events and end events while it can happen in at least $(m!)^2$ total orders.

Given the ordering conflict $c^{s}$ representing the concurrency of multiple tasks, we extract an implicit ordering relation $\varphi^{s}$. As $c^{s}$ should not hold in any consistent total order, all the partial orders of a consistent total order must entail its negation $\lnot c^{s}$, and thus $\varphi^{s} = \lnot c^{s}$ is a clause as follows:
\begin{equation}
\small
\label{eq:implicit_state}
    \varphi^{s} = \lnot \underset{i,j}{\land} (x_i^\vdash \prec x_j^\dashv) = \underset{i,j}{\lor} \lnot (x_i^\vdash \prec x_j^\dashv) = \underset{i,j}{\lor} (x_j^\dashv \prec x_i^\vdash).
\end{equation}
The intuitive explanation of $\varphi^s$ is that some task should end before others in a consistent total order.

\subsubsection{Temporal Inconsistency}
Ordering conflicts and implicit ordering relations can also be extracted from temporal inconsistency. A total order may violate temporal constraints such as temporal requirements (1) - (4) in our motivating example, even though we have extracted partial orders from these temporal requirements. For example, if we order the start of Flow-A to be after the ends of the other two flows, the mission horizon will exceed 70 seconds, which violates temporal requirement (4). As each flow takes at least 30 seconds, and Flow-B and Flow-C should end at least 20 seconds apart, the mission will take at least 80 seconds under this ordering, which exceeds 70 seconds. As Flow-A starts at 1, and Flow-B and Flow-C ends at 3 and 4 respectively, the ordering conflict of this temporal inconsistency is $c_6 = (3 \prec 1) \land (4 \prec 1)$. Then, the implicit ordering relation discovered from $c_6$ is the clause $\varphi_6 = \lnot ((3 \prec 1) \land (4 \prec 1)) = (1 \prec 3) \lor (1 \prec 4)$, which means Flow-A should begin before one other flow ends such that the whole mission takes less than 70 seconds.

We formally model the temporal requirements as a Temporal Constraint Network \cite{dechter1991temporal}. A total order on the events in the network is equivalent to imposing temporal constraints $(0,\infty)$ on every pair of events whose precedence relation is specified by this total order. As these imposed constraints tighten the network, temporal inconsistency may be introduced. In the distance graph form of this network, a temporal consistency checking algorithm is able to detect negative cycles that are composed of inconsistent temporal constraints. Given a negative cycle, we use a partial order $R^t_i = x_i^- \prec x_i^+$ to represent every temporal constraint that is added because of total ordering and involved in the cycle. The ordering conflict $c^t$ is used to represent this negative cycle as follows: 
\begin{equation}\small
    c^{t} = \underset{i}{\land}R^t_i = \underset{i}{\land}(x_i^- \prec x_i^+).
\end{equation}

Similar to extracting an implicit ordering relation from inconsistent concurrency, we obtain an implicit ordering relation $\varphi^t$ from $c^t$ by using $\varphi^t = \lnot c^t$ as follows:
\begin{equation}
\small
\label{eq:implicit_time}
    \varphi^{t} = \lnot \underset{i}{\land}(x_i^- \prec x_i^+) = \underset{i}{\lor}\lnot(x_i^- \prec x_i^+) = \underset{i}{\lor}(x_i^+ \prec x_i^-).
\end{equation}
The intuitive explanation of $\varphi^t$ is that some temporal constraint $(0,\infty)$ involved in the negative cycle should be removed to break this cycle.

\subsection{Resolving Ordering Conflicts}
In this section, we introduce the method to resolve ordering conflicts by finding the first order move that is able to jump over inconsistent total orders with respect to these conflicts.

Recall that $\Phi$ is a set of clauses, and thus we can check every clause against the current total order and determine all the unsatisfied clauses. An ordering conflict that is a conjunction of partial orders unsatisfying a clause can be obtained by negating this unsatisfied clause.

To resolve an ordering conflict, the search needs to move to a total order that negates at least one partial order in this ordering conflict. The first order move to achieve this negation is called the constituent kernel of this conflict. The kernel or combined kernel is the last order move of all the constituent kernels, which leads to a subtree where a consistent total order possibly exists with respect to all the ordering conflicts.

\begin{figure*}[ht]
\centering
\includegraphics[width=1.8\columnwidth]{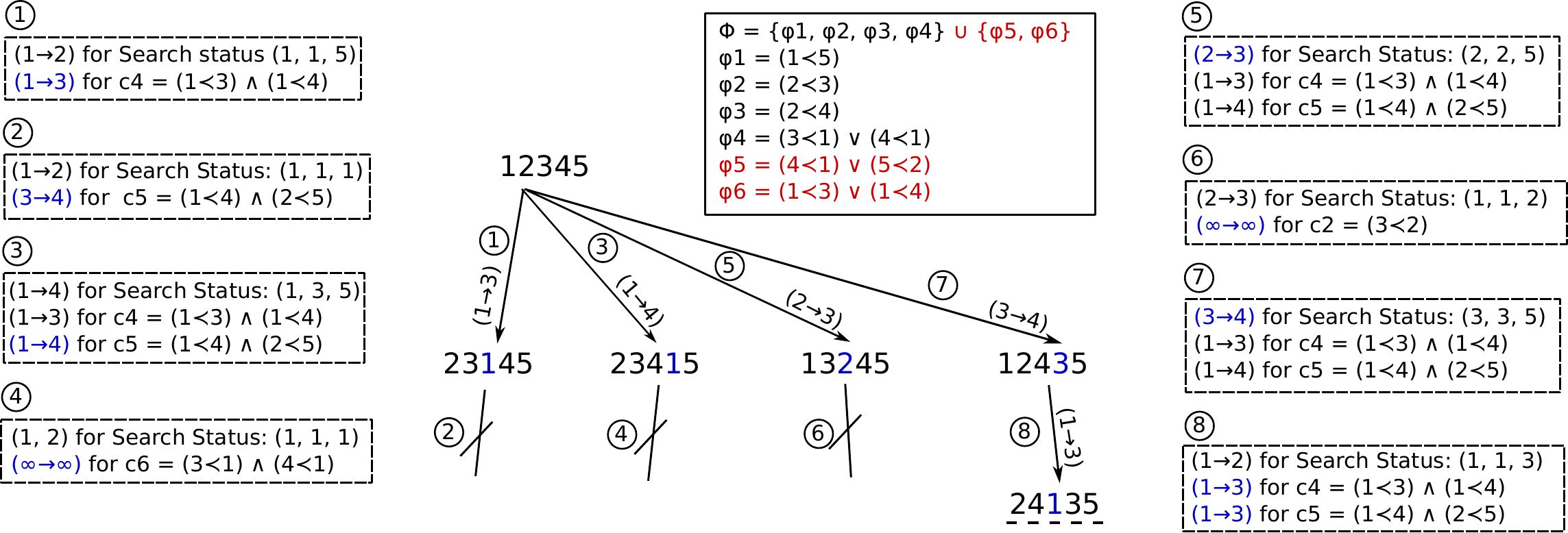}
\caption{Solving the motivating example by using CDITO in eight iterations. The procedures of computing combined kernels for these iterations are given in the dotted boxes; the combined kernels are in blue; the implicit ordering relations $\varphi_5$ and $\varphi_6$ are in red and discovered by the second and fourth iterations respectively.}
\label{fig:cdito}
\end{figure*}

\subsubsection{Constituent Kernel}
The constituent kernel of an ordering conflict from a total order is the first order move that leads to a total order negating at least one partial order in this ordering conflict. From another perspective, any order move that is before the constituent kernel would lead to an inconsistent total order. First consider the root total order $12345$ with the search status $(1,1,5)$ in our motivating example. Given the ordering relation $\Phi$, this total order violates the clause $\varphi_4 = (3 \prec 1) \lor (4 \prec 1)$, and the corresponding ordering conflict is its negation $c_4 = (1 \prec 3) \land (1 \prec 4)$. In order to resolve $c_4$, the next generated total order should negate $(1 \prec 3)$ or $(1 \prec 4)$. Recall that an order move $(i \rightarrow j)$ with smaller $(ni+j)$ is taken first with respect to our search strategy, and thus the search will first try $(1 \rightarrow 2)$. However, $(1 \rightarrow 2)$ will lead to a total order $21345$ and the ordering conflict $c_4$ remains. In addition, $21345$ has no child and it is impossible to resolve $c_4$ in its descendants. Therefore, we can safely skip $(1 \rightarrow 2)$ and apply $(1 \rightarrow 3)$ to generate $23145$ that resolves $c_4$. We call $(1 \rightarrow 3)$ the constituent kernel of $c_4$ from $12345$. For the sake of completeness, the constituent kernel should be the first order move in the search to resolve the ordering conflict even if the subsequent order moves $(1 \rightarrow 4)$ and $(1 \rightarrow 5)$ can also resolve $c_4$. We show an example of breaking completeness by taking order moves coming after $(1 \rightarrow 3)$. Assuming that the only consistent solution is $23415$ by taking $(1 \rightarrow 4)$, if we choose an order move that comes after $(1 \rightarrow 3)$ such as $(1 \rightarrow 5)$, we will skip the solution $23145$. By taking $(1 \rightarrow 3)$, we still have a chance to reach $23145$ by taking $(1 \rightarrow 4)$ as the next order move.

Note that the computation of constituent kernels depends on the current total order. Recall that an order move $(i \rightarrow j)$ means moving the $i^{\text{th}}$ event after the $j^{\text{th}}$ event in a total order instead of moving event $i$ after event $j$. As the constituent kernel is an order move that negates some partial orders of a conflict in the current total order, the constituent kernel varies when the current total order is different. In our motivating example, both $12345$ and $12435$ violate $\varphi_5 = (4 \prec 1) \lor (5 \prec 2)$ and share the conflict $c_5 = (1 \prec 4) \land (2 \prec 5)$. However, their constituent kernels are $(1 \rightarrow 4)$ and $(1 \rightarrow 3)$ respectively since $4$ is the fourth event in $12345$ but the third in $12435$. 

In addition to the current total order, the computation of a constituent kernel also considers the total order's level, which constrains the maximum event to move. Given a level $l$, we summarize constituent kernels to three types: (1) when resolving a conflict requires moving a event less than $l$, the order move will lead to its child. The above example of resolving the conflict $c_4$ in $12345$ is of this type; (2) when resolving a conflict requires moving the event $l$, the order move will lead straight to the current total order's sibling. For example, as $23145$ has the conflict $c_5 = (1 \prec 4) \land (2 \prec 5)$, we need to apply $(3 \rightarrow 4)$ to negate $(1 \prec 4)$, which leads to its sibling $23415$. 
(3) when resolving a conflict requires moving a event larger than $l$, no valid constituent kernel exists, and we denote it as $(\infty, \infty)$. For example, resolving the conflict $c_2 = (3 \prec 2)$ in $13245$ requires moving $3$. However, $3$ cannot be moved in the subtree rooted at $13245$ because the level of $13245$ is $2$.

Formally, consider a total order $\mathcal{L}=(p_1,p_2,...,p_n)$ with level $l$ and a corresponding ordering conflict $c_r = \underset{s \in S(r)}{\land} q_{rs}$, where $S(r)$ is the set of indices of the partial orders in $c_r$, and $q_{rs}$ is the $s^{\text{th}}$ partial order of $c_r$. We iterate over every $q_{rs} = (p_{i'} \prec p_{j'})$ for $s \in S(r)$ and then determine the first order move $(i^{\dagger}_r \rightarrow j^{\dagger}_r)$ in the search to negate some $q_{rs}$. This order move is the constituent kernel of the ordering conflict $c_r$ from $\mathcal{L}$, which is computed as follows:
\begin{equation}
\small
\label{eq:constituent_kernel}
    (i^{\dagger}_r \rightarrow j^{\dagger}_r) = \underset{(i' \rightarrow j') \in \Delta_r}{\text{argmin}} (ni'+j'),
\end{equation}
where $\Delta_r= \{(i' \rightarrow j') \, | \, ((p_{i'} \prec p_{j'}) \in c_r)  \land (p_{i'} \leq l)\} \cup \{(\infty \rightarrow \infty)\}$, and $(i' \rightarrow j')$ is the first subsequent order move to negate $q_{rs} = (p_{i'} \prec p_{j'})$. Recall that the order move with smaller $(ni'+j')$ is applied first. Therefore, the constituent kernel of $c_r$ is the order move $(i' \rightarrow j') \in \Delta_r$ with the smallest $(ni'+j')$, which is the first to negate at least one $q_{rs} = (p_{i'} \prec p_{j'})$. Note that $\Delta_r$ removes all the infeasible order moves that require moving the events larger than $l$ by forcing $(p_{i'} \leq l)$. If all the moves are infeasible, $\Delta_r = \{(\infty \rightarrow \infty)\}$ and $(i^{\dagger}_r \rightarrow j^{\dagger}_r) = (\infty \rightarrow \infty)$. Therefore, Equation~\ref{eq:constituent_kernel} is able to compute all the types of constituent kernels we just introduced.

\subsubsection{Combined Kernel}
Now we introduce how to combine multiple constituent kernels to compute a combined kernel for all the ordering conflicts.

As any order move before a constituent kernel is inconsistent, the combined kernel is the first order move that leads to a subtree where a consistent total order possibly exists. For example, the total order $12345$ violates $\varphi_4 = (3 \prec 1) \lor (4 \prec 1)$ and $\varphi_5 = (4 \prec 1) \lor (5 \prec 2)$, resulting in the ordering conflicts $c_4 = (1 \prec 3) \land (1 \prec 4)$ and $c_5 = (1 \prec 4) \land (2 \prec 5)$. The constituent kernels for $c_4$ and $c_5$ are $(1 \rightarrow 3)$ and $(1 \rightarrow 4)$ respectively. The combined kernel of $c_4$ and $c_5$ should be $(1 \rightarrow 4)$, which is the first subsequent move to resolve both conflicts. If we apply a move before this constituent kernel such as $(1 \rightarrow 3)$ and then obtain $23145$, we will notice that $23145$ still has the ordering conflict $c_5$.

Note that the computation of combined kernels also considers the last order move $(i \rightarrow j)$ that is stored in the search status $(i,j,l)$ along with the level $l$. Given a search status, any order move before the last order move $(i \rightarrow j)$ has been taken or skipped and thus is unavailable. Therefore, the constituent kernels of the ordering conflicts are probably before $(i \rightarrow j)$ and cannot be taken. In this case, we simply follow the standard total order search as shown in Lines~\ref{line:tos:next-move-init-start}-\ref{line:tos:next-move-init-end} in Algorithm~\ref{alg:tos}. For example, the total order $12345$ associated with search status $(2,2,5)$ has the conflicts $c_4 = (1 \prec 3) \land (1 \prec 4)$ and $c_5 = (1 \prec 4) \land (2 \prec 5)$, whose constituent kernels are $(1 \rightarrow 3)$ and $(1 \rightarrow 4)$ respectively. Since the search status is $(2,2,5)$, any order move $(i,j)$ with $(5i+j) < (5 \times 2 + 2) = 12$ has been checked to be inconsistent. Therefore, the next order move is $(2 \rightarrow 3)$, which follows the standard total order search. More generally, we can see the search status $(i,j,l)$ as a special ordering conflict such that a special constituent kernel $(i \rightarrow j+1)$ or $(i+1 \rightarrow i+2)$ is generated, before which all the order moves have been checked to be inconsistent by the previous search.

In addition to the aforementioned three types of constituent kernels, considering the search status introduces the fourth kernel type $(l \rightarrow l+1)$ for combined kernels. The kernel of this type is obtained by substituting $i = (l-1)$ into $(i+1 \rightarrow i+2)$, which means a total order's feasible children whose levels are less than $l$ have been exhausted. The method to handle this kernel along with other combined kernels is introduced in Section~\ref{section:approach:cdito}.

Consider a total order $\mathcal{L}$ with search status $(i,j,l)$ and ordering conflicts $\mathcal{C}$ = $\{ c_r = \lnot \varphi_r\, | \, \mathcal{L}  \text{ violates } \varphi_r \in \Phi  \}$, we first compute the constituent kernel $(i^{\dagger}_r \rightarrow j^{\dagger}_r)$ for every ordering conflict $c_r \in \mathcal{C}$. Then, the combined kernel $(i^{\dagger} \rightarrow j^{\dagger})$ of
$\mathcal{C}$ is obtained as follows:
\begin{equation}
\small
\label{eq:combined_kernel}
    (i^{\dagger} \rightarrow j^{\dagger}) = \underset{(i^{\dagger}_r \rightarrow j^{\dagger}_r) \in \Delta^\dagger}{\text{argmax}}(ni^{\dagger}_r + j^{\dagger}_r),
\end{equation}
where $\Delta^\dagger = \{(i^{\dagger}_r \rightarrow j^{\dagger}_r) \, | \, (i^{\dagger}_r \rightarrow j^{\dagger}_r)$ is the constituent kernel of $c_r \in \mathcal{C}\} \cup \{(i_0 \rightarrow j_0)\}$. Note that $(i_0 \rightarrow j_0)$ is $(i \rightarrow j+1)$ or $(i+1 \rightarrow i+2)$, which is the special constituent kernel considering the search status. As any order move before a constituent kernel leads to inconsistent total orders, the combined kernel obtained by Equation~\ref{eq:combined_kernel} is an order move before which all the order moves are inconsistent with respect to the search status and all the ordering conflicts.

Now we introduce \textsc{NextMove} (Algorithm~\ref{alg:next-move}) that computes the combined kernel $(i^{\dagger},j^{\dagger})$ for a set of ordering conflicts $\mathcal{C}$ of a total order $\mathcal{L}$ with search status $(i,j,l)$. The inner loop (Lines~\ref{line:next-move:inner-init}-\ref{line:next-move:inner-end}) computes the constituent kernel $(i^{\dagger}_r \rightarrow j^{\dagger}_r)$ for each ordering conflict $c_{r} = \underset{s \in S(r)}{\land} q_{rs}$ by following Equation~\ref{eq:constituent_kernel}, and the outer loop computes the combined kernel $(i^{\dagger} \rightarrow j^{\dagger})$ with respect to all the constituent kernels by following Equation~\ref{eq:combined_kernel}. In the inner loop, the algorithm begins with the infeasible order move (Line~\ref{line:next-move:inner-init}) and updates the kernel with a nearer order move that negates a partial order (Lines~\ref{line:next-move:inner-update-1}-\ref{line:next-move:inner-update-2}). To accelerate this procedure, Line~\ref{line:next-move:inner-break} breaks the inner loop when $(i^{\dagger}_r \rightarrow j^{\dagger}_r)$ would be entering an inconsistent subtree with respect to the incumbent of the combined kernel by following Equation~\ref{eq:combined_kernel}. In the outer loop, Lines~\ref{line:next-move:init-start}-\ref{line:next-move:init-end} initialize $(i^{\dagger} \rightarrow j^{\dagger})$ with the last order move $(i,j)$. The algorithm updates $(i^{\dagger} \rightarrow j^{\dagger})$ when a constituent kernel that jumps further is found (Line~\ref{line:next-move:outer-update}). To accelerate this procedure, we return the infeasible order move $(\infty,\infty)$ when an ordering conflict is unsolvable (Line~\ref{line:next-move-unsolvable}).

\begin{algorithm}[t]\small
\caption{\textsc{NextMove}} 
\label{alg:next-move}
\KwIn{$\langle \mathcal{L}, \mathcal{C}, i, j, l\rangle$ \tcp{$\mathcal{C} = \{ c_r =  \underset{s \in S(r)}{\land} q_{rs} \,|\, r \in R$\}}}
\KwOut{$(i^{\dagger}, j^{\dagger})$}
    \eIf{$j \not = n$ \label{line:next-move:init-start}}
    {$(i^{\dagger},j^{\dagger}) \gets (i,j + 1)$}
    {$(i^{\dagger},j^{\dagger}) \gets (i+1,i+2)$ \label{line:next-move:init-end}}
    \For{$r \in R$}{
        $(i^{\dagger}_{r},j^{\dagger}_{r}) \gets (\infty,\infty)$\; \label{line:next-move:inner-init}
        \For{$s \in S(r)$}{
            $(p_{i'} \prec p_{j'}) \gets q_{rs}$ in $\mathcal{L}$ \;
            \If{$(ni'+ j') \leq (ni^{\dagger} + j^{\dagger})$}
            {$(i^{\dagger}_r,j^{\dagger}_r) \gets (1, 1)$ and break \label{line:next-move:inner-break}}
            \If{$p_{i}' \leq l \text{ and } (ni' + j') < (ni^{\dagger}_{r} + j^{\dagger}_{r})$ \label{line:next-move:inner-update-1}}
            {$(i^{\dagger}_{r},j^{\dagger}_{r}) \gets (i',j')$}}\label{line:next-move:inner-update-2}\label{line:next-move:inner-end}
        \lIf{$p_{i^{\dagger}_r} > l$}{\KwRet{$(\infty,\infty)$}} \label{line:next-move-unsolvable} 
        \lIf{$(ni^{\dagger} + j^{\dagger}) < (ni^{\dagger}_{r} + j^{\dagger}_{r})$}
            {$(i^{\dagger},j^{\dagger}) \gets (i^{\dagger}_r,j^{\dagger}_r)$}}\label{line:next-move:outer-update}       
    \KwRet{$(i^{\dagger}, j^{\dagger})$}
\end{algorithm}

\subsection{CDITO Algorithm}\label{section:approach:cdito}
In this section, we present the CDITO algorithm (Algorithm~\ref{alg:cdito}). CDITO follows the same search strategy of Total Order Search (Algorithm~\ref{alg:tos}) within a total order tree and uses combined kernels to jump over inconsistent total orders (Algorithm~\ref{alg:next-move}).

Algorithm~\ref{alg:cdito} takes as input a total order $\mathcal{L}$, a stack of search statuses $\mathcal{P}$, an ordering relation $\Phi$, and a consistency function $h$. CDITO outputs either an empty set (Line~\ref{line:cdito:wipeout}) or a total order that satisfies $\Phi$ and is checked to be consistent by $h$ (Lines~\ref{line:cdito:order-cons}-\ref{line:cdito:return-l}).

\begin{algorithm}[t]
\caption{CDITO}\small
\label{alg:cdito}
\KwIn{$\langle \mathcal{L}, \mathcal{P}, \Phi, h \rangle$}
\KwOut{$\mathcal{L}$ or \{\} 
}
\While{$\Phi \text{ is consistent and } \mathcal{P} \not = \{\}$}{
    \If{$\mathcal{L} \text{ satisfies } \Phi$\label{line:cdito:order-cons}}
            {$\langle consistent?, \mathcal{C}_h \rangle \gets h(\mathcal{L})$\;
            \eIf{$ consistent? = \top$ \label{line:cdito:h-cons}}
                {\KwRet{$\mathcal{L}$} \label{line:cdito:return-l}}
                {$\Phi \gets \Phi \cup \{ \varphi_r = \lnot c_r \, | \, c_r \in \mathcal{C}_h  \}$\label{line:cdito:implicit}}}
    $( i,j,l ) \gets \mathcal{P}[1] $ \;    
    $\mathcal{C} \gets \{ c_r = \lnot \varphi_r\, | \, \mathcal{L} \text{ violates } \varphi_r \in \Phi  \}$ \; \label{line:cdito:conflicts}
    $(i^{\dagger},j^{\dagger}) \gets \textsc{NextMove}(\mathcal{L}, \mathcal{C}, i, j, l)$ \; \label{line:cdito:kernel}
    \eIf{$i^{\dagger} < l$ \label{line:cdito:handle-move-begin}}
        {$\mathcal{L} \gets \mathcal{L} \oplus (i^{\dagger} \rightarrow j^{\dagger})$\;\label{line:cdito:child_generation}
            $\mathcal{P}[1] \gets ( i^{\dagger}, j^{\dagger}, l )$\; 
            push $(1, 1, i^{\dagger})$ to $\mathcal{P}$ \; \label{line:cdito:child_status}}
        {pop $\mathcal{P}$\; \label{line:cdito:popstack}
        \If{$\mathcal{P} \not = \{\}$}{
            $(i',j',l') \gets \mathcal{P}[1] $ \;
            $\mathcal{L} \gets \mathcal{L} \oplus (j' \rightarrow (i' - 1))$\;\label{line:cdito:backinsert}
            \lIf{$l < i^{\dagger} < \infty$}
                {$\mathcal{P}[1] \gets (i',j^{\dagger}-1,l')$
                \label{line:cdito:sibling-move}}
            \lIf{$i^{\dagger} = \infty$}
                {$\mathcal{P}[1] \gets ((i' + 1),(i' + 1),l' )$ \label{line:cdito:prune-cluster}}
             \label{line:cdito:handle-move-end}}}}
\KwRet{\{\}}\; \label{line:cdito:wipeout}
\end{algorithm}

The consistency function $h$ takes as input a total order and determines grounded consistency other than $\Phi$. Implicit ordering relations can be extracted by negating the ordering conflicts $\mathcal{C}_h$ found by $h$ on demand, such as Equation~\ref{eq:implicit_state} and Equation~\ref{eq:implicit_time}. Then, these relations are added into $\Phi$ to avoid generating total orders with similar inconsistency (Line~\ref{line:cdito:implicit}). As shown in Figure~\ref{fig:cdito}, $h$ is invoked two times in the total eight iterations: the second iteration extracts an implicit ordering relation $\varphi_5$ from inconsistent concurrency, and the fourth iteration extracts $\varphi_6$ from temporal inconsistency.

In order to compute the combined kernel $(i^{\dagger} \rightarrow j^{\dagger})$, the algorithm collects all the ordering conflicts $\mathcal{C}$ (Line~\ref{line:cdito:conflicts}) and inputs $\mathcal{C}$ into \textsc{NextMove} (Algorithm \ref{alg:next-move}) along with the current total order $\mathcal{L}$ and the search status $(i,j,l)$ (Line~\ref{line:cdito:kernel}).

Recall that there are four types of combined kernel $(i^{\dagger} \rightarrow j^{\dagger})$, and they are handled differently (Lines \ref{line:cdito:handle-move-begin}-\ref{line:cdito:handle-move-end}): (1) $i^{\dagger} < l$; (2) $ l < i^{\dagger} < \infty$;(3) $i^{\dagger} = \infty$; and (4) $i^{\dagger} = l$. As shown in Figure~\ref{fig:cdito}, the motivating example is solved by 8 iterations: the second iteration is of type (2), the fourth and the sixth iterations are of type (3), and the other iterations are of type (1). When a kernel is of type (1), we directly take it to generate a child (Lines~\ref{line:cdito:child_generation}-\ref{line:cdito:child_status}). For the other types, CDITO backtracks (Lines \ref{line:cdito:popstack}-\ref{line:cdito:backinsert}) but updates the parent's search status in different ways. When a kernel is of type (4), the search status is not updated, which is the same as Lines~\ref{line:tos:popstack}-\ref{line:tos:backinsert} in Algorithm~\ref{alg:tos}. In the following two paragraphs, we will focus on type (2) and type (3).

When $l < i^{\dagger} < \infty$, the kernel is of type (2), and CDITO updates the parent's search status to $(i',j^{\dagger}-1,l')$ (Line~\ref{line:cdito:sibling-move}). This prunes all the descendants of the current total order, its siblings generated by taking order moves $(i' \rightarrow j'+1),(i' \rightarrow j'+2),..,(i' \rightarrow j^{\dagger}-1)$, and all the descendants of these siblings. Since this kernel requires moving $p_{i'}$ after $p_{j^{\dagger}}$, some conflicts exist in these siblings until $(i' \rightarrow j^{\dagger})$ is taken. By Lemma~\ref{lemma:fixed_in_children}, these conflicts also remains in the descendants of these siblings, and thus all these total orders can be safely pruned.

When $i^{\dagger} = \infty$, the kernel is of type (3), CDITO updates the parent's search status to $(i'+1,i'+1,l')$ (Line~\ref{line:cdito:prune-cluster}). This prunes all the descendants of the current total order, all its siblings with level $l$, and all the descendants of these siblings. Since there exist some ordering conflicts that require moving an event larger than $l$, these conflicts remain in its siblings with level $l$. By Lemma~\ref{lemma:fixed_in_children}, these conflicts also remain in the descendants of this total order and these siblings as well.

\section{Experiments}\label{section:experiment}
To evaluate the effect of incorporating the conflicts from state and time for the total order search, we benchmarked CDITO against ITO \cite{wang2015} on temporal network configuration problems with different complexity and size. These problems involve routing flows and allocating bandwidth resources with respect to requirements on loss, delay, bandwidth, and deadlines. Note that our motivating example is also a temporal network configuration problem.

The problems were provided by a communication network simulator that generates network flows with random duration constraints and characteristic requirements on a meshed network. The simulator setup is as follows: (1) the mission horizon is 300s; (2) the meshed network has 16 nodes and 240 links; (3) the required loss, delay, and bandwidth of each link are uniformly generated from continuous intervals [0.1,0.3]\%, [0.1,0.3]s, and [500,1000]kbps; (4) the loss, delay, throughput, minimum duration of each network flow are uniformly generated from  [0.1,0.3]\%, [0.1,0.3]s, [600,1000]kbps, and [20,80]s; (5) the generator adds temporal constraints between randomly chosen events with a duration (0,100], and the number of temporal constraints is one fifth of the number of flows.

We combined two sub-solvers to construct the consistency function $h$: (1) we used a CP solver in \cite{chen2018radmax} to reason over routing and bandwidth allocation with respect to loss, delay, and bandwidth constraints; (2) Incremental Temporal Consistency \cite{shu2005enabling} was used to check the temporal consistency of plans.

We tested five scenarios of 10, 20, 30, 40, and 50 flows with CDITO and ITO. We ran 100 trials for each scenario, and the timeout, which is the duration between two replan requests in real-world experimental devices, was 20 seconds.

\begin{table}[htb]\scriptsize
\centering
\begin{tabular}{|c||c|c|c|c|c|c|}
\hline
\multirow{2}{*}{\textbf{\#flows}} & \multicolumn{3}{l|}{\hspace{0.09\columnwidth}\textbf{CDITO}} & \multicolumn{3}{l|}{\hspace{0.11\columnwidth}\textbf{ITO}} \\ 
\cline{2-7} 
& \textbf{\#solved} & \textbf{$N_S$}  & \textbf{$N_U$} & \textbf{\#solved} & \textbf{$N_S'$} & \textbf{$N_U'$} \\ \hline
10 & 94 & 7 & 221 & 9 & 1 & 233 \\ \hline
20 & 91 & 6 & 27 & 5 & 1 & 38 \\ \hline
30 & 86 & 5 & 10 & 6 & 1 & 18 \\ \hline
40 & 82 & 4 & 16 & 11 & 1 & 21 \\ \hline
50 & 74 & 5 & 14 & 8 & 1 & 24 \\ \hline
\end{tabular}
\caption{Experimental results. \textbf{\#solved}: number of solved trials; \textbf{$N_S$}, \textbf{$N_U$}: average number of calls to $h$ in solved and unsolved trails by using CDITO; \textbf{$N_S'$}, \textbf{$N_U'$}: average number of calls to $h$ in solved and unsolved trails by using ITO.}
\label{table:result}
\end{table}

Table~\ref{table:result} shows that CDITO solves most of the problems, while ITO solves few in 20 seconds. It can be seen that CDITO finds consistent solutions quickly after calling $h$ around five times in all the solved trials, which demonstrates that CDITO is capable of using conflicts to efficiently guide search and avoid unnecessary order generation or consistency check. However, $N_S'$ equals 1 in all the solved trials, which means, in large-scale problems, ITO can find solutions only if a good initial order is generated. Overall, as \#flows increases, checking grounded consistency is more expensive. Thus, both methods generate fewer orders in unsolved trials, and we observe the significant decreases of $N_U$ and $N_U'$. In every scenario, $N_U$ is slightly less than $N_U'$, which demonstrates that efforts put on reasoning over conflicts are not expensive compared to other costs. Note that, as CDITO prunes a large portion of total orders, CDITO goes further within the total order tree than ITO with the same number of order generations.

\section{Conclusion}
In this paper, we presented CDITO, a systematic and incremental algorithm that efficiently orders the events in a partially ordered plan by applying conflict-directed search on a tree of total orders. Given the ordering conflicts discovered in the search, CDITO is able to generate resolutions to skip inconsistent total orders by exploiting the special structure of the total order tree. During the search, our method also extracts implicit ordering relations from grounded consistency check such as state and temporal consistency. CDITO thus avoids unnecessary and expensive state and temporal consistency check. This is supported by our experiments on temporal network configuration problems generated by a communication network simulator, which empirically demonstrate the efficiency of CDITO over ITO.

\subsubsection{Acknowledgements.} This project was funded by the Defense Advanced Research Projects Agency under grant Contract No. HR0011-15-C-0098.

\bibliographystyle{aaai}
\small{\bibliography{bib}}
\end{document}